\documentclass[11pt]{article}

\usepackage{misc/acl2023}
\usepackage{hyperref}

\usepackage{times}
\usepackage{comment}
\usepackage{xfrac}
\usepackage{xcolor}
\usepackage{latexsym}
\usepackage{enumerate}
\usepackage[inline]{enumitem}
\setenumerate[1]{label=(\arabic*)}
\usepackage{listings}
\usepackage{array}
\usepackage{algorithm, algorithmicx, algpseudocode}

\newcommand{\algcolorC}{brown!40}

\usepackage[T1]{fontenc}
\usepackage[utf8]{inputenc}
\usepackage{inconsolata}
\usepackage{tikz}
\usetikzlibrary{arrows}

\usepackage{amsfonts}
\usepackage{amsmath}
\usepackage{amssymb,bm}
\usepackage{amsthm}
\usepackage[frozencache,cachedir=.]{minted}
\usepackage{minted}
\usepackage{enumitem}
\usepackage{algpseudocode}
\usepackage{microtype}
\usepackage{ulem}
\usepackage{booktabs}
\usepackage{multicol,multirow}
\usepackage{soul} 
\usepackage{thmtools}
\usepackage{thm-restate}

\DeclareMathOperator*{\argmax}{argmax}
\newcommand{\bigO}[1]{\mathcal{O}\left(#1\right)}

\newcommand{\spacesymbol}{\textbf{\_}}

\usepackage[capitalize,noabbrev]{cleveref}
\crefname{section}{\S}{\S\S}
\crefname{table}{Tab.}{}
\crefname{figure}{Fig.}{}
\crefname{algorithm}{Alg.}{}
\crefname{equation}{Eq.}{Eq.}
\crefname{definition}{Definition}{Definition}
\crefname{appendix}{App.}{}
\crefname{theorem}{Theorem}{}
\crefname{myexample}{Example}{}
\DeclareCaptionType{myexamplef}[Example][List of Examples]
\crefname{myexamplef}{Example}{}
\DeclareCaptionType{mycode}[Code][List of Codes]
\crefname{mycode}{Code}{}
\crefname{mynote}{Note}{}
\crefname{prop}{Proposition}{}
\crefname{cor}{Corollary}{}
\crefname{observation}{Observation}{}
\crefname{assumption}{Assumption}{}
\crefname{hypothesis}{Hyp.}{Hypotheses}
\crefformat{section}{\S#2#1#3}
\newcommand{\defeq}[0]{\mathrel{\stackrel{\textnormal{\tiny def}}{=}}}

\theoremstyle{plain}
\newtheorem{theorem}{Theorem}[section]
\newtheorem{proposition}[theorem]{Proposition}
\newtheorem{myexample}[theorem]{Example}

\newtheorem{lemma}[theorem]{Lemma}

\newtheorem{definition}[theorem]{Definition}

\theoremstyle{remark}

\newcommand{\defn}[1]{\textbf{#1}}

\usepackage{todonotes}

\makeatletter
\newcommand*\iftodonotes{\if@todonotes@disabled\expandafter\@secondoftwo\else\expandafter\@firstoftwo\fi} 
\makeatother

\newcommand{\Leftarrowbracketed}{\textbf{($\Leftarrow$):}}
\newcommand{\Rightarrowbracketed}{\textbf{($\Rightarrow$):}}
\newcommand{\mop}[2]{[#1,#2]}
\newcommand{\mopt}[2]{[#1,#2]}

\newcommand{\alphabet}{\Sigma}
\newcommand{\mergealphabet}{\Upsilon_{\alphabet}}

\newcommand{\str}{\boldsymbol{x}}

\newcommand{\enc}{{\small \textsc{Apply}}}

\newcommand{\yield}{{\small \textsc{Yield}}}
\newcommand{\vnu}{\boldsymbol{\nu}}
\newcommand{\merge}{\mu}
\newcommand{\merges}{\boldsymbol{\merge}}

\newcommand{\mergesgreedy}{\merges^{\dagger}}
\newcommand{\mergegreedy}{\merge^{\dagger}}
\newcommand{\mergesopt}{\merges^\star}
\newcommand{\mergeopt}{\merge^\star}
\newcommand{\mergeemph}[1]{\colorbox{gray!30}{#1}}
\newcommand{\sigmaopt}{{\sigma(\mergesopt)}}
\newcommand{\sigmaopthat}{{\hat{\sigma}(\mergesopt)}}
\newcommand{\sigmaoptgreedy}{{\sigma'(\mergesopt, \mergesgreedy)}}
\newcommand{\sigmaoptgreedyhat}{{\hat{\sigma}'(\mergesopt, \mergesgreedy)}}

\newcommand{\validmerges}{\mathcal{M}_{\mergealphabet}}

\newcommand{\prooftext}[1]{\text{\color{black!50}#1}}

\newcommand{\naturals}{\mathbb{N}}

\definecolor{ethblue}{rgb}{0.0,0.0,0.0}
\newcommand{\ethletter}{
    \hspace{-0.5mm}\text{
    \fontfamily{phv}\fontseries{bx}\fontsize{7}{\baselineskip}\selectfont
    \textit{\textbf{\color{ethblue}{E}}}}
}
\definecolor{jhublue}{rgb}{0,0.1,0.4}
\newcommand{\jhuletter}{
    \hspace{-0.5mm}\text{
    \fontfamily{phv}\fontseries{bx}\fontsize{7}{\baselineskip}\selectfont
    \textbf{\color{jhublue}{J}}}
}
\newcommand{\hrefEmail}[2]{\href{mailto:#1}{\color{black}{#2}}}

\title{A Formal Perspective on Byte-Pair Encoding}

\author{
Vilém Zouhar$^{\ethletter}$ \quad Clara Meister$^{\ethletter}$ \quad Juan Luis Gastaldi$^{\ethletter}$ \quad Li Du$^{\jhuletter}$ \\
\bf Tim Vieira$^{\jhuletter}$ \quad Mrinmaya Sachan$^{\ethletter}$ \quad Ryan Cotterell$^{\ethletter}$ \\
\text{} \\
ETH Z{\"u}rich$^{\ethletter}$ \quad Johns Hopkins University$^{\jhuletter}$ \\
\texttt{\{\hrefEmail{vzouhar@ethz.ch}{vzouhar},\hrefEmail{meistecl@ethz.ch}{meistecl},\hrefEmail{gjuan@ethz.ch}{gjuan},\hrefEmail{msachan@ethz.ch}{msachan},\hrefEmail{ryan.cotterell@ethz.ch}{ryan.cotterell}\}@ethz.ch} \\
\texttt{\{\hrefEmail{leodu@cs.jhu.edu}{leodu},\hrefEmail{timv@cs.jhu.edu}{timv}\}@cs.jhu.edu}
}

\begin{document}
\maketitle
\begin{abstract}


Byte-Pair Encoding (BPE) is a popular algorithm used for tokenizing data in NLP, despite being devised initially as a compression method.
BPE appears to be a greedy algorithm at face value, but the underlying optimization problem that BPE seeks to solve has not yet been laid down.
We formalize BPE as a combinatorial optimization problem.
Via submodular functions, we prove that the iterative greedy version is a $\frac{1}{\sigmaopt}(1-e^{-\sigmaopt})$-approximation of an optimal merge sequence, where $\sigmaopt$ is the total backward curvature with respect to the optimal merge sequence $\mergesopt$.
Empirically the lower bound of the approximation is $\approx 0.37$.

We provide a faster implementation of BPE which improves the runtime complexity from $\bigO{N M}$ to $\bigO{N \log M}$, where $N$ is the sequence length and $M$ is the merge count.
Finally, we optimize the brute-force algorithm for optimal BPE using memoization.


\end{abstract}

\section{Introduction}

Byte-Pair Encoding (BPE) is a popular technique for building and applying an encoding scheme to natural language texts.
It is one the most common tokenization methods used for language models \citep{radford2019language,bostrom2020byte,gpt3,scao2022bloom} as well as for various other conditional language modeling tasks, e.g., machine translation \citep{ding2019call} and chatbots \citep{zhang2020dialogpt}.
Despite having been popularized by \citet{sennrich-etal-2016-neural} in NLP as a tokenization scheme, BPE has its roots in the compression literature, where
\citet{gage1994new} introduce the method as a faster alternative to Lempel--Ziv--Welch \cite[13.4]{welch1984technique,Cover2006}.
However, the ubiquity of BPE notwithstanding, the formal underpinnings of the algorithm are underexplored, and there are no existing proven guarantees about BPE's performance.\looseness=-1

The training and applying of BPE are traditionally presented as greedy algorithms, but the exact optimization problems they seek to solve are neither presented in the original work of \citet{gage1994new} nor in the work of \citet{sennrich-etal-2016-neural}.
We fill this void by offering a clean formalization of BPE training as maximizing a function we call compression utility\footnote{How much space the compression saves (\Cref{defn:compression}).} over a specific combinatorial space, which we define in \Cref{def:valid}.
Unexpectedly, we are then able to prove a bound on BPE's approximation error using total backward curvature $\sigmaopt$ \citep{zhang2015string}.
Specifically, we find the ratio of compression utilities between the greedy method and the optimum is bounded below by $\frac{1}{\sigmaopt}(1 - e^{-\sigmaopt})$, which we find empirically $\approx 0.37$ for $\widehat{\sigma}(\mergesopt) = 2.5$.
Our proof of correctness hinges on the theory of submodular functions \cite{krause14survey,bilmes2022submodularity}.\footnote{The proof further relies on a specific property of problem which BPE optimizes that we term hierarchical sequence submodularity. Hierarchical sequence submodularity neither follows from nor implies sequence submodularity, but, nevertheless, bears some superficial similarity to sequence submodularity---hence, our choice of name.} 
Indeed, we are able to prove that compression utility is a special kind of submodular function \cite{malekian2009combinatorial} over a constrained space.
And, despite the presence of the length constraint, which we expound upon formally in \cref{sec:greedy}, we are able to prove a similar bound to $1-\sfrac{1}{e}$ as in the unconstrained case \cite{alaei2010maximizing}.\looseness=-1

Additionally, we give a formal analysis of greedy BPE's runtime and provide a speed-up over the original implementation \cite{gage1994new,sennrich-etal-2016-neural}.
Our runtime improvement stems from the development of a nuanced data structure that allows us to share work between iterations of the greedy procedure and that lends itself to an amortized analysis. 
Specifically, given a string with $N$ characters with a desired merge count of $M$ (usually $N \gg M$), our implementation runs in $\bigO{N \log M}$, an improvement over the $\bigO{N M}$-time algorithm presented by \citet{sennrich-etal-2016-neural} and the $\bigO{N \log N}$ analysis presented by \citet{kudo2018sentencepiece}.
Finally, our formalism allows us to construct an exact program for computing an optimal solution to the BPE training problem. 
Unfortunately, the algorithm runs in exponential time, but it is still significantly faster than a naïve brute-force approach.

Our work should give NLP practitioners confidence that BPE is a wise choice for learning a subword vocabulary based on compression principles.
In general, such constrained submodular maximization problems are hard \cite{lovasz1983submodular}.
While we do not have a proof that the BPE problem specifically is NP-hard, it does not seem likely that we could find an efficient algorithm for the problem.
Regarding the runtime, our implementation of greedy BPE runs nearly linearly in the length of the string which would be hard to improve unless we plan to not consider the entire string.

\section{Formalizing Byte-Pair Encoding}

We first provide a brief intuition for the BPE training problem and the
greedy algorithm that is typically employed to solve it.
Then, we will develop a formalization of BPE using the tools of combinatorial optimization, rather than as a procedure.\footnote{Note that we are interested in the optimality of the algorithm for creating the subword vocabulary in terms of compression and not the optimality of the encoding in terms of coding-theoretic metrics such as \emph{efficiency}. This aspect of BPE is explored by \citet{tokenization_noiseless}.}

\subsection{A Worked Example}
\begin{table}[htbp]
\vspace{2mm}
\centering
\resizebox{\linewidth}{!}{
\begin{tabular}{>{\small\bf}cl}
\toprule
{merge 1} & \texttt{\mergeemph{p i} c k e d \, \mergeemph{p i} c k l e d \, \mergeemph{p i} c k l e s} \\
{merge 2} & \texttt{pi \mergeemph{c k} e d \hspace{3.5mm} pi \mergeemph{c k} l e d \hspace{2.5mm} pi \mergeemph{c k} l e s} \\
{merge 3} & \texttt{\mergeemph{pi ck} e d \hspace{4.1mm} \mergeemph{pi ck} l e d \hspace{4.3mm} \mergeemph{pi ck} l e s} \\
{merge 4} & \texttt{pick \mergeemph{e d} \hspace{7.0mm} pick l \mergeemph{e d} \hspace{6.2mm} pick l e s} \\
{merge 5} & \texttt{pick ed \hspace{10.0mm} \mergeemph{pick l} ed \hspace{8.3mm} \mergeemph{pick l} e s} \\
{final} & \texttt{pick ed \hspace{10.9mm} pickl ed \hspace{12.2mm} pickl e s} \\
\bottomrule
\end{tabular}
}
\captionof{myexamplef}{Compression of the text \textit{picked pickled pickles} using 5 greedy merges according to the greedy BPE algorithm. The most frequently occurring pair of vocabulary items is highlighted and subsequently merged.
The merge sequence is $\langle$\textit{\mop{p}{i}, \mop{c}{k}, \mop{pi}{ck}, \mop{e}{d}, \mop{pick}{l}}$\rangle$ (notation simplified for clarity).}
\label{tab:micro_bpe_example}
\end{table}

Consider the string in \Cref{tab:micro_bpe_example}: \textit{picked pickled pickles}. 
We wish to create a compact representation of this string, where compactness is quantified in terms of the number of symbols (i.e., vocabulary units) required to precisely encode the string. 
The free parameter is the vocabulary that we will use to construct this representation, albeit the total size of the chosen vocabulary is often a constraint.\footnote{We require a unique encoding for each item, which implies that encoding size will be dependent on the total number of vocabulary items (e.g. the dimension of a one-hot encoding or the number of bits required to encode the text).\looseness=-1}
In our example, let's assume we are allowed a maximum number of 13 symbols in the vocabulary\footnote{Typically, all the symbols in $\alphabet$ are part of the vocabulary so that all texts can be represented, even with lower efficiency.} with which we can encode our string. 
The question is: ``How can we select these symbols to achieve our goal of compactness under this constraint?''

Let us first consider the simple choice of using all the characters present in the string as our vocabulary: This scheme leads to a representation with a length of 22 units, including spaces.
In order to decrease this length (while retaining all information present in the original string), we would need to add an additional symbol to our vocabulary: one with which we can replace co-occurrences of two symbols. 
But how should we choose this entry? 
One strategy---the one employed by the BPE algorithm---is to use the concatenation of the adjacent units $a\,\,b$ that occur with the highest frequency in our string; all occurrences of these adjacent units could then be replaced with a single new unit $ab$. 
We refer to this as a \defn{merge}, which we later define and denote formally as $\mop{a}{b}$.
In Example~\ref{tab:micro_bpe_example}, the first merge is $\mop{p}{i}$, and leads to a representation of length 19 with vocabulary size of 9+1.
We can iteratively repeat the same process; the application of 5 total merges results in the vocabulary units \textit{pick}, \textit{pickl}, \textit{ed}, \textit{e}, and \textit{s}.
These \defn{subwords}\footnote{The term \defn{subword} corresponds to a merge yield (\Cref{def:yield}). We use `subword' and `merge' interchangeably.} allow us to represent our original string using just 9+1 symbols.
If we continued merging, the text representation would become shorter (in terms of number of symbols required to create the representation) but the merge count (and vocabulary size) would grow. 
Therefore, the number of merges $M$, or also the merge count, is a hyperparameter to the whole procedure.
The procedure outlined above is exactly the greedy algorithm for BPE proposed by \citet{gage1994new}.
We provide a minimal implementation in Python in \Cref{code:minimal_bpe}. 

We will define the compression gain of a merge at any given step of the algorithm, corresponding to the number of occurrences where a merge can be applied.
The compression gain of a merge does not always correspond to the frequency of adjacent merge components in that string, due to possible overlaps.
Consider, for instance, the string \textit{aaa} and the merge $\mop{a}{a}$.
The frequency of $aa$ is 2, but the merge can be applied only once ($\mopt{a}{a}a$).
While \citet{gage1994new} and \citet{sennrich-etal-2016-neural} admit overlapping pair counts, \citet{kudo2018sentencepiece}'s popular implementation adjusts the algorithm to disregard the overlaps.
We stick to the latter, which is more suitable from the optimization standpoint adopted here.


\begin{figure}[htbp]
\centering
\begin{minted}[fontsize={\fontsize{9.5}{9}\selectfont},linenos,xleftmargin=7mm]{python}
from collections import Counter
from typing import Union, Tuple, List

def bpe(xs: Union[str, List], V: int):
  for _ in range(V):
    pairs = Counter(zip(xs, xs[1:]))
    top_pair = pairs.most_common(1)[0][0]
    xs = merge(list(xs), top_pair)
  return xs

def merge(xs: List, pair: Tuple):
  ys = []
  while xs:
    if tuple(xs[:2]) == pair:
      ys.append(pair)
      xs = xs[2:]
    else:
      ys.append(xs.pop(0))
  return ys 
\end{minted}
\captionof{mycode}{A minimal implementation of \citeposs{sennrich-etal-2016-neural} greedy algorithm for BPE in Python. See \Cref{code:fixed_bpe} for a version with overlap-adjusted counts.}
\label{code:minimal_bpe}
\end{figure}

\subsection{Merges}\label{sec:merges}
The fundamental building block of the BPE problem is a merge, which we define formally below.
Informally, a merge is the action of creating a new symbol out of two existing ones.
Out of convention, we also refer to the resulting object as a merge.

\begin{definition}
\label{definition:merge}
Let $\alphabet$ be an alphabet, a finite, non-empty set.
The set of all \defn{merges} over $\alphabet$ is the smallest set of pairs $\mergealphabet$ with the following closure property:
\begin{itemize}[noitemsep,topsep=0mm]
\item $\sigma \in \alphabet \Longrightarrow \sigma \in \mergealphabet$ (called \defn{trivial merges});
\item $\merge', \merge'' \in \mergealphabet \Longrightarrow \mop{\merge'}{\merge''} \in \mergealphabet$
\end{itemize}
where we denote the non-trivial elements of $\mergealphabet$ as $\merge=\mop{\merge'}{\merge''}$.
A \defn{merge sequence} is a sequence of merges, which we denote $\merges = \langle \merge_1, \ldots, \merge_N \rangle \in \mergealphabet^*$.\footnote{$(\cdot)^*$ is the Kleene closure.}\looseness=-1
\end{definition}
\noindent It is perhaps easiest to understand the concept of a merge through an example.\looseness=-1
\begin{myexample}\label{ex:basic-merges}
Given the alphabet $\alphabet = \{a, b, c\}$, the following are some of the elements of $\mergealphabet$\textup{:} $\mop{a}{b}$, $\mop{a}{\mop{a}{b}}$, and $\mop{\mop{a}{b}}{\mop{a}{c}}$.
We obtain a merge sequence by arranging these merges into an ordering $\merges = \langle \mop{a}{b}, \mop{a}{\mop{a}{b}}, \mop{\mop{a}{b}}{\mop{a}{c}} \rangle \in \mergealphabet^*$.\looseness=-1
\end{myexample}

Note that the strings corresponding to the merges in a merge sequence---along with the characters that make up the set of trivial merges---determine a \defn{vocabulary}, to be used in downstream applications.\footnote{I.e., the size of the vocabulary is $|\merges| + |\alphabet|$.}
The greedy BPE algorithm constructs a merge sequence iteratively by picking each merge as the pairing of neighbouring symbols in the current sequence of symbols that is being processed.
For instance, the sequence $\merges$ in \Cref{ex:basic-merges} is not valid since it does not contain the merge $\mop{a}{c}$ before the third element $\mop{\mop{a}{b}}{\mop{a}{c}}$.

\begin{definition}\label{def:valid}
    We define a merge sequence $\merges = \langle \merge_1, \ldots, \merge_N \rangle \in \mergealphabet^*$ to be \defn{valid} if, for every $ \merge_n $, it holds that $ \merge_n = \mop{\merge'}{\merge''}$, where for $\merge \in \{\merge',\merge''\}$, $\merge = \merge_k$ with $k<n$, or $\merge \in \alphabet$. We denote the set of valid merge sequences $\validmerges$.
\end{definition}

Note that $\validmerges$ is closed under concatenation, i.e., for two valid merge sequences $\merges', \merges'' \in \validmerges$, we have that
$\merges'\merges'' \in \validmerges$,\footnote{The merge sequence can contain the same merges multiple times and still be valid. Only the later occurrences of the merge will not reduce the representation size.}
where we use $\merges\merges'$ to denote the sequence concatenation of $\merges$ and $\merges'$.

\begin{figure}[htbp]
\centering
\begin{tikzpicture}[
    scale=0.75,
    char/.style={
        text height={height("a")+0pt},
        text width={width("a")+0pt},
        inner sep=1mm,
        align=center,
        font=\fontsize{11}{0}\selectfont\itshape,
    },
    subword/.style={
        align=center,
        font=\fontsize{11}{0}\selectfont,
    },
    merge/.style={
        draw,
        circle,
        align=center,
        line width=0.35mm,
        inner sep=0.7mm,
    },
    connector/.style={
        line width=0.4mm
    }
] 
\node[char] at (0, 0) (r0c0) {a};
\node[char] at (1, 0) (r0c1) {b};
\node[char] at (2, 0) (r0c2) {a};
\node[char] at (3, 0) (r0c3) {a};
\node[char] at (4, 0) (r0c4) {b};
\node[char] at (5, 0) (r0c5) {a};
\node[char] at (6, 0) (r0c6) {c};
\node[char] at (7, 0) (r0c7) {b};
\node[char] at (8, 0) (r0c8) {c};
\node[char] at (9, 0) (r0c9) {b};

\node[merge] at (0.5, 1.1) (r1c0) {$\merge_1$};
\node[merge] at (3.5, 1.1) (r1c1) {$\merge_1$};
\node[merge] at (6.5, 1.1) (r1c2) {$\merge_2$};
\node[merge] at (8.5, 1.1) (r1c3) {$\merge_2$};
\node[merge] at (1.25, 2.2) (r2c0) {$\merge_3$};
\node[merge] at (4.25, 2.2) (r2c1) {$\merge_3$};
\node[merge] at (5.375, 3.3) (r3c0) {$\merge_4$};

\draw[connector] (r1c0) -- (r0c0);
\draw[connector] (r1c0) -- (r0c1);
\draw[connector] (r2c0) -- (r1c0);
\draw[connector] (r2c0) -- (r0c2);

\draw[connector] (r1c1) -- (r0c3);
\draw[connector] (r1c1) -- (r0c4);
\draw[connector] (r1c2) -- (r0c6);
\draw[connector] (r1c2) -- (r0c7);
\draw[connector] (r2c1) -- (r1c1);
\draw[connector] (r2c1) -- (r0c5);
\draw[connector] (r3c0) -- (r2c1);
\draw[connector] (r3c0) -- (r1c2);

\draw[connector] (r1c3) -- (r0c8);
\draw[connector] (r1c3) -- (r0c9);

\node[subword] at (1, -0.7) {$\mopt{\mopt{a}{b}}{a}$};
\node[subword] at (5, -0.7) {$\mop{\mopt{\mopt{a}{b}}{a}}{\mopt{c}{b}}$};
\node[subword] at (8.5, -0.7) {$\mopt{c}{b}$};
\end{tikzpicture}

\vspace{-2mm}
\caption{Application of the merge sequence $\merges = \langle \mop{a}{b}, \mop{c}{b}, \mop{\mop{a}{b}}{a}, \mop{\mop{\mop{a}{b}}{a}}{\mop{c}{b}}\rangle$ on the string $\str = abaabacbcb$. The result can be represented as an ordered forest. Each tree is associated with a subword in the text: $aba$, $abacb$, and $cb$.}
\label{fig:merge_tree_example}
\end{figure}

\paragraph{Applying Merge Sequences.}

Given some string $\str\in\alphabet^*$, we can derive the representation of that string according to the merge sequence $\merges = \langle \merge_1, \ldots, \merge_N\rangle$ by iteratively \defn{applying} each merge $\merge_n$. Note that by the definition of $\mergealphabet^*$, we can trivially lift a string $\str = \langle\sigma_1, \sigma_2, \ldots\rangle$ to a merge sequence by treating each of its characters $\sigma_i\in\alphabet$ as merges. Thus, we  define this procedure more generally in terms of some arbitrary $\overline\merges \in \mergealphabet^*$.  Concretely, we denote the application of a merge $\merge_n$  to $\overline\merges$  as $\enc_{\merge_n}(\overline\merges)$. 
As suggested by \cref{code:minimal_bpe} (line 11), this action consists of replacing all $\overline\merge_k, \overline\merge_{k+1}$ in $\overline{\merges}$ such that $(\overline{\merge}_{k},\overline{\merge}_{k+1})=\merge_n$ by $\merge_n$ itself, in a left-to-right fashion. 
We thus obtain a new $\overline{\merge} \in \mergealphabet^*$, to which a new single merge can be applied.
We lift $\enc$ to a merge sequence $\merges$  by simply repeating the application of $\enc$ on $\overline{\merges}^{(n)}$ for the successive $\merge_n$ in $\merges$; accordingly, we denote this procedure as $\enc_{\merges}(\overline{\merges})$. 
As a result, we obtain $\overline{\merges}^{(|\merges|)}$, which is a non-overlapping ordered forest, i.e., a partial bracketing of the original string $\str$.
We provide an example in \Cref{fig:merge_tree_example}.
Note that the application of the merge sequence is deterministic.\looseness=-1


\paragraph{String Yields.}
We now define a conceptually reverse operation to applying merges, i.e., deriving a string from structured $\overline{\merges}^{(n)}$.
\begin{definition}
\label{def:yield}
The \defn{yield} of a single $\overline\merge \in \mergealphabet$, denoted as $\yield(\overline\merge)$, is defined recursively:
\begin{align}
\hspace{-1mm}\yield(\overline\merge) =
\begin{cases}
\yield(\overline\merge') \yield(\overline\merge'') & \hspace{-3mm} \textbf{if }\overline \merge=\mop{\overline\merge'}{\overline\merge''} \\
\merge & \hspace{-3mm} \textbf{if } \overline\merge \in \alphabet
\end{cases}
\end{align}
\end{definition}
\noindent As an example, $\yield(\mop{\mop{a}{a}}{\mop{\mop{c}{b}}{c}})$ is $aacbc$. For a given $\overline\merges$, $\yield$ is applied sequentially. The resulting characters can then be concatenated to derive a single string. The yield operation can also be used to derive vocabulary units---often referred to as subwords; explicitly, the yields of individual merges in a sequence $\merges$ can be used to form a vocabulary.

Strictly speaking, in \citeposs{sennrich-etal-2016-neural} implementation of BPE, the elements of the merge sequences $ \merges $ are not of the form $ \merge_n = \mop{\merge'}{\merge''} \in \mergealphabet $, but rather $ \merge_n = \mop{\yield(\merge')}{\yield(\merge'')} \in \alphabet^*\times\alphabet^* $, i.e., rather than consisting of prior merges as in our formalization, the merges of \citeposs{sennrich-etal-2016-neural} consist of the yields of those merges.
This introduces an ambiguity with respect to our formalization since: for a given merge sequence in that implementation, more than one sequence $ \merges \in \mergealphabet^* $ could correspond, some of which would not be valid.
As an example, consider the sequence $ \langle \mop{a}{b}, \mop{ab}{c}, \mop{abc}{d} \rangle $ which could correspond to either $ \langle \mop{a}{b}, \mop{\mop{a}{b}}{c}, \mop{\mop{\mop{a}{b}}{c}}{d} \rangle $ or $ \langle \mop{a}{b}, \mop{\mop{a}{b}}{c}, \mop{\mop{a}{\mop{b}{c}}}{d} \rangle $, the last of which is invalid.
However, it turns out that this is not an issue for us: by construction, the successive elements of the sequence are determined by the previous ones (cf. \Cref*{algo:iterative_greedy_bpe_slow}), which means that, in fact there is no ambiguity, and the merge sequences in \citeposs{sennrich-etal-2016-neural} implementation always correspond to what our formalization defines as a valid merge sequence.




\subsection{The BPE Training Optimization Problem}
We now define the BPE training task as a combinatorial optimization problem.
The objective we seek to optimize is the compression utility of the chosen merge sequence (taken with respect to a string), which we define below.\looseness=-1
\begin{definition}
\label{defn:compression}
Let $\str \in \alphabet^*$ be a string.
We define the \defn{compression utility} of a valid merge sequence $\merges$ applied to $\str$ as the following function:
\begin{equation}
\kappa_{\str}(\merges) = |\str| - |\enc_{\merges}(\str)|
\end{equation}
Note that for any merge sequence $\merges$, $\kappa_{\str}(\merges) \geq 0$ and we take $\kappa_{\str}(\langle \rangle) = 0$.
Then, for any merge sequence $\merges' = \langle \merge'_1, \ldots, \merge'_{|\str|-1} \rangle$ of length $|\str|-1$  where every merge produces replacements, we have $\kappa_{\str}(\merges') = |\str|-1$ (see proof of \Cref{thm:faster_runtime}). 
\end{definition}

We can further define the \defn{compression gain} of two merge sequences with respect to each other.
\begin{definition}
\label{def:compression_gain_utility}
The \defn{compression gain} of $\merges'$ with respect to a sequence $\merges$, denoted as $\kappa_{\str}(\merges' \mid \merges)$, is defined as
\begin{align}\kappa_{\str}(\merges  \merges') - \kappa_{\str}(\merges).
\end{align}
Similarly, the compression gain of a single merge $\merge$ with respect to a sequence $\merges$, denoted as $\kappa_{\str}(\merge \mid \merges)$, is defined as $\kappa_{\str}(\merges  \merge) - \kappa_{\str}(\merges)$.
\end{definition}

We use the compression gain to later make a sequence of observations which leads to proving the function submodularity and eventually its approximation bound of the BPE training algorithm.\looseness=-1


\begin{algorithm}[t]
{\fontsize{11}{12}\selectfont
\begin{algorithmic}[1]
\State $\merges \gets \langle \rangle$
\For{$i \text{ in } \{0,\ldots,M\}$}
    \State
        $\merge \gets \displaystyle\argmax_{(\merge', \merge'') \in \text{set}(\str)^2} 
        \Call{PairFreq}{\str, (\merge',\merge'')}$
        \label{algo:iterative_greedy_bpe_slow_line_pairs}
    \State
        $\str \gets \Call{Apply}{\merge,\str}$
        \label{algo:iterative_greedy_bpe_slow_line_merge}
    \State $\merges \gets \merges \circ \langle \merge \rangle$
\EndFor
\State \Return $\merges, \str$
\end{algorithmic}
}
\caption{%
    Iterative Greedy BPE (slow).\newline
    \textbf{Inputs}: sequence $\str$, merge count $M$\newline
    \textbf{Output}: merge sequence $\merges$, tokenized sequence $\str$ \newline
    \textsc{PairFreq} are non-overlapping pair frequencies
}
\label{algo:iterative_greedy_bpe_slow}
\end{algorithm}
Now, armed with \cref{defn:compression}, we can formally
state our optimization problem.
In words, \textit{we seek to find a valid merge sequence $\merges$ with length of $M$
that maximizes the compression utility $\kappa_{\str}(\cdot)$ for a pre-specified string $\str \in \alphabet^*$}. 
We write this combinatorial optimization problem more formally as follows:\footnote{In practice, it does not happen that $|\str| < M$ and so we use $|\mergesopt| = M$ for convenience instead of $|\mergesopt| \leq M$.}
\begin{equation}\label{eq:optimization-objective}
\mergesopt = \argmax_{\substack{\merges \in \validmerges\\ |\merges| = M}} \kappa_{\str}(\merges)
\end{equation}
The most common procedure found in the NLP literature for solving \cref{eq:optimization-objective} is a greedy algorithm \cite{gage1994new,sennrich-etal-2016-neural}.
The implementation of \citeposs{gage1994new} algorithm presented by \citet{sennrich-etal-2016-neural} runs in $\bigO{N M}$ time ($N = |\str|, M = |\mergesopt|$). 
We describe this greedy algorithm in detail in \Cref{sec:greedy} and provide a novel theoretical result: \emph{The algorithm comes with a $\frac{1}{\sigmaopt}(1 - e^{-\sigmaopt})$ bound on its approximation error of \cref{eq:optimization-objective}.}
In \Cref{sec:speed-up}, we further offer an asymptotic speed-up to \citeposs{sennrich-etal-2016-neural} algorithm, reducing its runtime to $\bigO{N \log M}.$
Finally, for completeness, we offer an exact program for finding an optimal valid merge sequence in \Cref{sec:dynamic-program}.
While this algorithm runs in exponential time, which prevents it to be used in real applications, it is still faster than the brute-force counterpart.


\section{A Greedy Approximation of BPE}
\label{sec:greedy}

We demonstrate that, for any string $\str \in \alphabet^*$, the following bound holds
\begin{equation}
\frac{\kappa_{\str}(\mergesgreedy)}{\kappa_{\str}(\mergesopt)} \geq \frac{1}{\sigmaopt}(1 - e^{-\sigmaopt})
\end{equation}
 where, as in the previous section,  $\mergesgreedy$ is the valid merge sequence output by the greedy algorithm and $\mergesopt$ is an optimal valid merge sequence. 
To prove this bound, we rely heavily on the theory of submodularity \citep{krause14survey,bilmes2022submodularity}.\looseness=-1

\subsection{Properties of Compression Utility ($\kappa$)}
We start by proving some useful facts about the compression utility function $\kappa_{\str}$.
Specifically, we first show that $\kappa_{\str}$ is a specific type of monotone non-decreasing submodular sequence function, which we make precise in the following definitions.\looseness=-1
\begin{definition}\label{def:monotone}
A real-valued function $f$ over valid merge sequences is \defn{monotone non-decreasing} if, for all $\merges \in \validmerges$ and for all $n \in \naturals$, it holds that $f(\merges_{< n}) \geq f(\merges_{<n-1})$, where $\merges_{< n} \defeq \langle\merge_1, \ldots, \merge_{n-1}\rangle$.\looseness=-1
\end{definition}

\begin{proposition}
Let $\kappa_{\str}$ be the compression utility function.
Then, for a fixed $\str \in \alphabet^*$, $\kappa_{\str}(\cdot)$ is monotone (\Cref{def:monotone}).
\end{proposition}

\begin{proof}
For all $n \in \naturals$, we have that $\kappa_{\str}(\merges_{<n}) =\allowbreak \kappa_{\str}(\merges_{<n-1}) + \underbrace{\kappa_{\str}(\merge_n \mid \merges_{<n-1})}_{\geq 0}$.
It follows that $\kappa_{\str}(\cdot)$ is monotone non-decreasing.\looseness=-1
\end{proof}
Next, we turn to a definition of sequence submodularity from \citet{alaei2010maximizing}.
In contrast to \citeposs{alaei2010maximizing} definition, we add the additional constraint that a merge-sequence function must take a \emph{valid} merge sequence as an argument.
\begin{definition}\label{def:submodular}
A real-valued function $f$ over valid merge sequences is \defn{submodular} if, for all  $\merges, \merges' \in \validmerges$ such that $\merges' \preccurlyeq \merges$,\footnote{I.e., we have that $\merges'$ is a prefix of $\merges$.} and for all $\nu \in \mergealphabet$ such that both $\merges' \nu$ and $\merges \nu$ are valid, we have
\begin{equation}
\label{eq:function_submodularity}
f(\nu \mid \merges') \geq f(\nu \mid \merges).
\end{equation}
\end{definition}

\begin{restatable}{proposition}{compressionutilitysubmodular}
Let $\kappa_{\str}$ be the compression utility function.
Then, for a fixed $\str \in \alphabet^*$, $\kappa_{\str}(\cdot)$ is submodular (\Cref{def:submodular}) when the domain is restricted to the set of valid merges $\validmerges$.\looseness=-1
\end{restatable}
\begin{proof}
Let $\merges, \merges' \in \validmerges$ such that $\merges' \preccurlyeq \merges$, and let $\nu = [\nu',\nu'']$ be any merge such that $\merges\nu, \merges'\nu \in \validmerges.$
First, notice that, once a merge $\merge_n$ in a merge sequence $\merges$ is applied, the number of occurrences of $\merge_n$ in $\kappa_{\str}(\merges_{\leq n})$ cannot be increased by any sequence of further applications,
because all submerges of $\merge_n$ where applied exhaustively (i.e., to all consecutive occurrences of their immediate submerges).
Now, from $\merges' \nu \in \validmerges$, it follows that both $\nu'$ and $\nu''$ are in $\merges'$.
Therefore, the number of occurrences $\nu'$ and $\nu''$, and \emph{a fortiori} of successive occurrences of them, cannot be greater in $\kappa_{\str}(\merges)$ than in $\kappa_{\str}(\merges')$,
and hence $\kappa_{\str}(\nu \mid \merges) \leq \kappa_{\str}(\nu \mid \merges')$,
which proves the submodularity of $\kappa_{\str}$ over $\validmerges$.
\end{proof}

In the context of compression, the submodularity property means, that the compression gain achieved after adding a specific merge to a merge sequence can never increase with merge sequence length.
However, the requirement that the added merge does not create an invalid merge sequence is important.
We highlight this importance in the following example.


\begin{myexample}
Consider $\alphabet = \{a, b, c, d, e\}$, the string $\str=aabcde$, and the valid merge sequences $\merges' = \langle \mop{a}{a} \rangle$ and $\merges = \langle \mop{a}{a}, \mop{c}{d} \rangle$.
Note that $\merges' \preccurlyeq \merges$.
These merge sequences have compression utilities $\kappa_{\str}(\merges') = 6 - 5 = 1$ and $\kappa_{\str}(\merges) = 6 - 4 = 2$, respectively.
Next, consider the merge sequence $\vnu = \langle \mop{b}{\mop{c}{d}}, \mop{\mop{b}{\mop{c}{d}}}{e}\rangle$.
Now, $\kappa_{\str}(\vnu \mid \merges') = 0$ and $\kappa_{\str}(\vnu \mid \merges) = 2$, which violates submodularity because $\merge' \preccurlyeq \merges$.
What went wrong?
The problem is that $\merges \vnu$ is not a \emph{valid} merge sequence.\looseness=-1
\end{myexample}

In order to formally prove our desired guarantee regarding the approximation bound of the greedy BPE algorithm, it is not enough that compression utility is sequence submodular over valid merge sequences.
For this reason, we identified another property of the compression utility function that allows us to push through our result.\looseness=-1

\begin{definition}\label{defn:subset-parital-order}
We define the following \defn{partial order on merges}.
For merges $\merge, \merge' \in \mergealphabet$, we say $\merge' \subset \merge$ iff 
$\merge'$ is a submerge of $\merge$.
The merge $\merge'$ is a \defn{submerge} of $\merge = \mop{\merge_1}{\merge_2}$ iff:
\begin{itemize}[noitemsep,topsep=0mm]
\item $\merge_1 = \merge'$,  or, $\merge_2 = \merge'$, or
\item $\merge' \subset \merge_1$, or $\merge' \subset \merge_2$.
\end{itemize}
\end{definition}

\begin{restatable}{definition}{hierarchicalsub}\label{def:hierarchical-submodular}
A real-valued function over valid merge sequences is \defn{hierachically sequence submodular} if, for every valid merge sequence of the form $\merges'\nu'\merges\nu$
where $\nu' \subset \nu$ according to the partial order given in \Cref{defn:subset-parital-order}, we have that
\begin{equation}\label{eq:hierarchical-submodularity}
f(\nu'\mid \merges'  ) \geq f(\nu \mid \merges'  \nu'  \merges).
\end{equation}
\end{restatable}
Note that hierarchical sequence submodularity is a different concept from function modularity, described in \Cref{def:submodular}.
Indeed, in the case of functions over valid merge sequences, neither submodularity nor hierarchical sequence submodularity implies the other.
To see this, note that roughly speaking, submodularity describes the difference in the value of a function when the same element is given as an argument, albeit conditioned on the presence of two different (but related) other arguments.
However, if the same argument is considered in \Cref{eq:hierarchical-submodularity}, we have
\begin{equation}
\kappa_{\str}(\nu' \mid \merges') \geq \kappa_{\str}(\nu' \mid \merges' \nu' \merges) = 0,
\end{equation}
which is a trivial bound due to the non-negativity of $\kappa_{\str}(\cdot)$.
The naming is inspired by the fact we require the partial order over merges, which creates the hierarchy. 


\begin{proposition}\label{prop:hierarchical-submodular}
Let $\kappa_{\str}$ be the compression utility function.
Then, for a fixed $\str \in \alphabet^*$, $\kappa_{\str}(\cdot)$ is hierarchically submodular (\Cref{def:submodular}) when the domain is restricted to the set of valid merges $\validmerges$.
\end{proposition}

\begin{proof}
Let $\str \in \alphabet^*$ be a string and $\merges$, $\merges'$ be valid merge sequences.
Furthermore, let $\nu$, $\nu'$ be merges such that $\nu' \subset \nu$ and $\merges' \nu' \merges \nu$ is itself a valid merge sequence.
Combinatorially, $\kappa_{\str}(\nu \mid \merges' \nu' \merges)$ is the number of replacements made in $\str$ by the single merge of $\nu$, after applying $\merges' \nu' \merges$.
However, since $\nu' \subset \nu$,
every new tree in $\str$ resulting from that by applying $\nu$ must have $\nu'$ as a descendant.
Thus, $\kappa_{\str}(\nu' \mid \merges' )$, which is
the number of new nodes in the forest created by applying $\nu'$, must be at least equal to $\kappa_{\str}(\nu \mid \merges' \nu' \merges)$, if not greater.
\end{proof}
\Cref{prop:hierarchical-submodular} gives us a different notion of submodularity, which is important for the proof of the greedy BPE training guarantee.
As an illustrative example of the proposition, we return to \Cref{fig:merge_tree_example}.
In this case, $\merges' = \langle \mop{a}{b} \rangle$, $\nu' = \mop{\mop{a}{b}}{a}$, $\merges = \langle \mop{c}{b} \rangle$, $\nu = \mop{\mop{\mop{a}{b}}{a}}{\mop{c}{b}}$.
Clearly, $\nu' \subset \nu$ and $\nu'$ appears twice, while $\nu$ only once.

Finally, we adapt the definition of total backward curvature from \citep{zhang2015string} to our needs.
Intuitively, the total backward curvature is related to how much the utility of $\merges$ can decrease if $\nu$ is applied before, at the beginning.
\begin{definition}\label{def:total_backward_curvature}
The \defn{total backward curvature} of the compression utility function $\kappa$ with respect to an optimal merge sequence $\mergesopt$ is denoted with $\sigmaopt$:
\begin{align}
\sigmaopt = \max_{\substack{\merges \in \mergealphabet^* \\ |\merges|\leq M}} \left\{ 1-\frac{\kappa(\merges \mergesopt)-\kappa(\mergesopt)}{\kappa(\merges)} \right\}\,.
\end{align}
\end{definition}

\subsection{The Greedy Algorithm for BPE}
In words, the greedy algorithm proceeds as follows:
For each of the $M$ iterations, the algorithm chooses the next merge that is both valid and (locally) maximizes the objective in \cref{eq:optimization-objective}.
We give pseudocode in \Cref{algo:iterative_greedy_bpe_slow}.
In practice, as shown in \Cref{code:minimal_bpe}, this is done by choosing the merge that occurs most frequently (can be adjusted for pair overlaps).
The main loop occurs $M$ times. 
In the subsequent theorem we show the approximation bound for the greedy algorithm.

\begin{restatable}{theorem}{greedyalmostoptimal}
\label{threorem:greedy_almost_optimal}
The greedy algorithm for BPE training, i.e., for learning a length $M$ merge sequence $\mergesgreedy$, is $\big(\frac{1}{\sigmaopt}(1-e^{-\sigmaopt})\big)$-optimal: for every string $\str \in \alphabet^*$
\begin{equation}
\frac{\kappa_{\str}(\mergesgreedy)}{\kappa_{\str}(\mergesopt)} \geq \frac{1}{\sigmaopt}(1 - e^{-\sigmaopt})
\end{equation}
with respect to the optimal length $M$ merge sequence $\mergesopt$.
\end{restatable}
\begin{proof}
The proof is shown in \Cref{proof:greedy_almost_optimal}.
\end{proof}

\subsection{Measuring Total Backward Curvature}

We do not have a formal bound for $\sigmaopt$ and estimate it by enumerating all strings of maximum length $|\str| \leq 15$ given a finite alphabet $|\alphabet|=5$ and maximum merge sequence size $|\mergesopt| < 5$. 
The found maximum is $\sigmaopthat = 2.5$, from which  follows an optimality bound of $\approx 0.37$. 
When we restrict our search to texts from a natural language (English), we obtain a slightly lower estimate $\sigmaopthat_N = 2.0$ and hence optimality bound $\approx 0.43$.
We leave the further study of the backward curvature constant to future work.

Notice that in the main proof of \Cref{threorem:greedy_almost_optimal} in \Cref{proof:greedy_almost_optimal}, we used $\sigma$ to bound only one particular type of sequence that becomes the prefix to $\mergesopt$, namely $\mergesgreedy$.
We may then check for prefixing only greedy sequences instead of taking the maximum across $\merges \in \mergealphabet^*, |\merges|\leq M$ as in \Cref{def:total_backward_curvature}:
\begin{align}
\hspace{-2mm}\sigmaoptgreedy = \left\{ 1-\frac{\kappa(\mergesgreedy_{<M} \mergesopt)-\kappa(\mergesopt)}{\kappa(\mergesgreedy_{<M})} \right\}
\end{align}
This yields $\sigmaoptgreedyhat = 1.5$ and therefore the bound of $\approx 0.52$.
More important than the particular bound value is that it is constant and that the BPE training algorithm can not be arbitratily suboptimal with sequence length.


\begin{table}
\centering
\resizebox{\linewidth}{!}{
\begin{tabular}{lcc}
\toprule
\textbf{Sequence}\hspace{-1cm} & \textbf{Pair frequencies} \\
\midrule
\textbf{Greedy}\hspace{-2cm} \\
\mop{a}{b}a\mop{a}{b}baa & \mergeemph{ab}: 2, ba: 2, aa: 2, bb: 1  \\
\mop{\mop{a}{b}}{a}\mop{a}{b}baa & \hspace{-5mm} \mergeemph{\mop{a}{b}a}: 1, \mop{a}{b}b: 1, ba: 1, aa:1, \mop{a}{\mop{a}{b}}: 1 \\
\textbf{Optimal}\hspace{-2cm} & \\
a\mop{b}{a}ab\mop{b}{a}a & {ab: 2}, \mergeemph{ba}: 2, aa: 2, bb: 1 \\
a\mop{\mop{b}{a}}{a}b\mop{\mop{b}{a}}{a} & {ab: 2}, a\mop{b}{a}: 1, \mergeemph{\mop{b}{a}a}: 2, b\mop{b}{a}: 1 \\
\bottomrule
\end{tabular}
}
\captionof{myexamplef}{In case of $\str = abaabbaa$ the greedy BPE yields a suboptimal compression utility (5 vs 4 subwords). Highlighted pairs show which one was chosen.}
\label{example:greedy_suboptimal}
\vspace{2mm}
\end{table}


\begin{algorithm}[htbp]
{\fontsize{11}{12}\selectfont
\begin{algorithmic}[1]
\State $\merges \gets \langle \rangle$
\State $\str \gets \Call{LinkedList}{\str}$
\State $h \gets \Call{MaxHeap}{\textsc{Pairs}(\str)}$
\For{$i \text{ in } 0 .. M$}
    \State $pos \gets h.\Call{top}{}$
    \For{$(w_1, w_2) \text{ in } pos$}
    \label{algline:iterative_greedy_bpe_faster_inner_1}
    
    \State $h.$\Call{RemovePosition}{$w_1.\textit{prev}, w_1$}
    \State $h.$\Call{RemovePosition}{$w_2, w_2.\textit{next}$}
    \State $w_1.\textit{val} \gets w_1.\textit{val} + w_2.\textit{val}$
    \State $w_1.\textit{next} \gets w_2.\textit{next}$
    \State $w_2.\textit{next}.\textit{prev} \gets w_1$
    
    \State $h.$\Call{AddPosition}{$w_1.\textit{prev}, w_1$}
    \State $h.$\Call{AddPosition}{$w_1, w_1.\textit{next}$}
    
    \EndFor
    \label{algline:iterative_greedy_bpe_faster_inner_2}
    \State $\merges \gets \merges \circ \langle \merge \rangle$
\EndFor
\State \Return $\str, \merges$
\end{algorithmic}
}
\caption{%
    Iterative Greedy BPE (faster).\newline
    \textbf{Inputs}: string $\str$, merge count $M$\newline
    \textbf{Output}: tokenized string $\str$, merge sequence $\merges$
}
\label{algo:iterative_greedy_bpe_faster}
\end{algorithm}

\begin{figure*}[htbp]
\centering
\includegraphics[width=\linewidth]{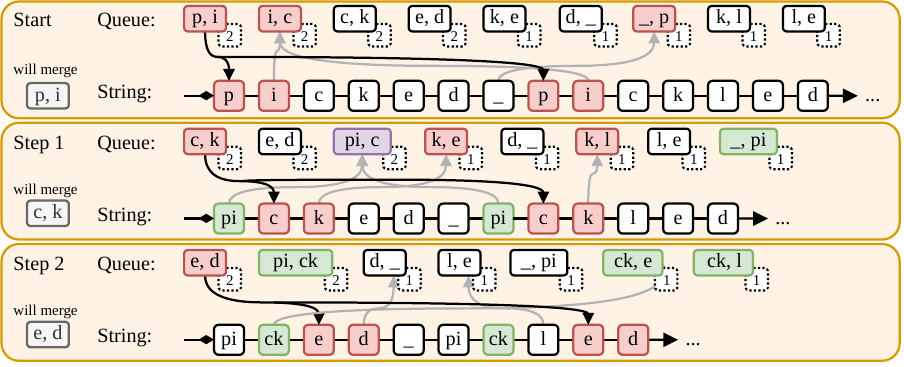}
\vspace{-8mm}
\caption{Visualization of linked list representation of the string and the associated priority queue (frequency values in dashed boxes) with merges. Nodes in red will be removed in the next step, nodes in green were added in contrast to the previous step and nodes in purple were just added but will be removed. Black lines from queue to the string show which nodes to merge. Grey lines show which pairs in the priority queue will have reduced frequencies.}
\label{fig:faster_merge_relink}
\end{figure*}

\section{A Runtime Speed-up}\label{sec:speed-up}
We now introduce a speed-up of the greedy BPE algorithm.
Assuming constant-time comparison of strings, finding the maximum pair count over the whole string is $\bigO{N}$, which is the same as applying one merge.
Therefore, this implementation has a runtime complexity of $\bigO{NM}$.
A large amount of time in the slow BPE implementation, presented by \citet{sennrich-etal-2016-neural} and shown in \Cref{algo:iterative_greedy_bpe_slow}, is spent on
\begin{enumerate*}
\item recalculating the frequencies of pairs (\Cref{algo:iterative_greedy_bpe_slow}, line \ref{algo:iterative_greedy_bpe_slow_line_pairs}) which are not affected by the most recent merge, and
\item scanning the whole string to apply a single merge (\Cref{algo:iterative_greedy_bpe_slow}, line \ref{algo:iterative_greedy_bpe_slow_line_merge}).
\end{enumerate*}
To make this explicit, consider the following example.\looseness=-1
\begin{myexample}
Consider $\str = abba (cddc)^n$ and merge $\mop{a}{b}$ for $n \geq 1$.
We can only apply the merge at the beginning of the string, which results in the forest $\mop{a}{b}ba (cddc)^n$.
However, \Cref{algo:iterative_greedy_bpe_faster} still scans the entirety of the sequence to recalculate the pair frequencies of $\mop{c}{d}, \mop{d}{c}$ and $\mop{c}{c}$.
This additional work is unnecessary.\looseness=-1
\end{myexample}

Our idea to speed up \Cref{algo:iterative_greedy_bpe_slow} stems from the insight that we do not have to iterate over the entire sequence, an $\bigO{N}$ operation, on each of the $M$ iterations.\footnote{$N$ is the string length $|\str|$ and $M$ the number of merges.}
Indeed, on the $t^{\text{th}}$ iteration, we show that one only has to do work proportional to the number of new nodes that are added to the forest (\Cref{algo:iterative_greedy_bpe_faster}, line~\ref{algline:iterative_greedy_bpe_faster_inner_1}).
To achieve this, we introduce a more efficient data structure for BPE.\footnote{\citet{kudo2018sentencepiece} make a similar observation, however, we prove a tighter bound on the runtime.}
Our first step is to treat the string as a linked list of subwords, initialized as a linked list of characters, that we destructively modify at each iteration.
With each possible merge, we store a list of pointers where the merge operation could happen.
The max heap is then sorted by the size of the sets.
Lines \ref{algline:iterative_greedy_bpe_faster_inner_1} to \ref{algline:iterative_greedy_bpe_faster_inner_2} in \Cref{algo:iterative_greedy_bpe_faster} show the necessary operations needed to be performed on the linked list.
Notably \textsc{RemovePosition} removes the specific pair position from the set in the max heap and \textsc{AddPosition} adds it.

See \Cref{fig:faster_merge_relink} for an illustration of applying a single merge in one place based on the introductory example in \Cref{tab:micro_bpe_example}.
The possible merge pairs are stored with a priority queue with their frequency as the sort key.
During one operation, we need to remove the top merge pair and add counts for the newly created possible merge.
The cost of one merge then becomes $\bigO{R_t \log M}$ where $R_t$ is the number of pairs in the string where the merge occurs and $\log M$ the complexity of adding and updating frequency of a new merge pair.
Note that it is not $\log N$, because we are keeping only top-$M$ possible pairs in the heap.

At first glance, this suggests the overall runtime of $\bigO{\sum_{t=1}^M R_t  \log M}$ with the worst case of the merge being applied along the whole string, therefore $\bigO{M N \log M}$.

\begin{restatable}{theorem}{fasterruntime}
\label{thm:faster_runtime}
Let $N$ be the length of the string $\str \in \alphabet^*$ that is given as input.
Then, \Cref{algo:iterative_greedy_bpe_faster} runs in $\bigO{N \log M}$ time.
\end{restatable}
\begin{proof}
Let $R_t$ be the amount of work performed at each iteration modifying the data structure.
We additionally do $\bigO{\log M}$ work updating the priority queue on lines~\ref{algline:iterative_greedy_bpe_faster_inner_1} to \ref{algline:iterative_greedy_bpe_faster_inner_2} in \Cref{algo:iterative_greedy_bpe_faster} since it has at most $M$ elements.
Thus, \Cref{algo:iterative_greedy_bpe_faster} clearly runs in $\bigO{\sum_{t=1}^M R_t \log M}$.
We perform an amortized analysis.
For this, we first make an observation about the upper bound on the number of merges and then show amortized analysis.
However, for a string $\str$ of length $N$, there are at most $N-1$ merges that can be applied to $\str$.
This implies that 
$\sum_{t=1}^M R_t \leq N$. 
Thus, $\bigO{\sum_{t=1}^M R_t \log M} = \bigO{N \log M}$, which proves the result.
\end{proof}





\section{An Exact Algorithm}\label{sec:dynamic-program}

In this section, we turn to developing an algorithm for exactly solving the BPE problem, i.e.,  \cref{eq:optimization-objective}.
We change algorithmic paradigms and switch to memoization.
While we are not able to devise a polynomial-time scheme, we are able to find an exact algorithm that is, in some cases, faster than the brute-force technique of enumerating all valid merge sequences. 
We first analyze the brute-force method of enumerating all valid merge sequences.

\begin{restatable}{proposition}{validmergessize}
\label{prop:validmerges_size}
The set of valid merges of length $M$ over a string $\str \in \alphabet^*$ is $\bigO{\min\left(|\Sigma|^{2M}, N^M\right)}$.
\end{restatable}
\begin{proof}
The proof can be found in \Cref{sec:proofs}.
\end{proof}

A simple direct enumeration of all possible merge sequences with the time complexity of one merge $\bigO{NM}$ gives us a brute-force algorithm that runs in $\bigO{NM \min\left(|\Sigma|^{2M}, N^M\right)}$ time. 
The brute-force program explores all possible sequences of merges---including many that are redundant.
For instance, both $\langle \mop{p}{o}, \mop{h}{a} \rangle$ and $\langle \mop{h}{a}, \mop{p}{o}\rangle$ induce the same partial bracketing when applied to another merge sequence, as in \cref{sec:merges}.
Luckily, we are able to offer an exact characterization of when two merge sequences induce the same bracketing. 
To this end, we provide the following definitions.
We use the term \defn{transposition} to refer to the swapping of items; i.e., a transposition $(i,j)$ over a merge sequence $\merges$ refers to the swapping of $\merge_i$ and $\merge_j$. 

\begin{definition}
A pair of merges $\merge = \mop{\merge_{n}}{\merge_{m}}$ and $\merge' = \mop{\merge_{n'}}{\merge_{m'}}$ \defn{conflicts} if for a symbol $a \in \alphabet$ and strings $\str, \str' \in \alphabet^*$, the yield of $\mop{\merge_{n}}{\merge_{m}}$ is $\str a$ and $\mop{\merge_{n}'}{\merge_{m}'}$ is $a \str'$.
\end{definition}
\begin{definition}\label{def:safe}
A transposition $(i, j)$ is \defn{safe} if and only if, for all $k< j$, $\merge_{k}$ does not conflict with $\merge_j$ and, for all $k>i$, $\merge_{k}$ does not conflict with $\merge_i$. 
A permutation $\pi = \langle \rho_1\rho_2\cdots\rho_n \rangle$, decomposed into transpositions, that maps one valid merge sequence $\merges$ to another valid merge sequence $\pi(\merges) = \merges'$ is \defn{safe} if and only if all transpositions are safe.
\end{definition}
\noindent Informally, \cref{def:safe} says that for a permutation to produce a valid merge sequence, there should be no conflicts between the swapped merges and all merges in between.
For example, given the merge sequence $\merges = \langle \mop{a}{b}, \mop{d}{d}, \mop{c}{a} \rangle$, the permutation $\pi = \langle (1, 3) \rangle$ would not be safe.

The reason for this definition is that safe permutations characterize when two merge sequences always give the same results.
Indeed, for $\str = ddabcacab$, applying the first merge sequence: $\enc_{\merges}(\str) = \mop{d}{d}\mop{a}{b}\mop{c}{a}\mop{c}{a}b$.
In contrast, applying the permuted one gives an alternative outcome: $\enc_{\pi(\merges)}(\str) = \mop{d}{d}\mop{a}{b}\mop{c}{a}c\mop{a}{b}$.

\begin{definition}
Two merge sequences $\merges$ and $\merges'$ are \defn{equivalent} if and only if, for all $\str \in \alphabet^*$, $\enc_{\merges}(\str) = \enc_{\merges'}(\str)$.
Symbolically, we write $\merges \equiv \merges'$ if $\merges$ and $\merges'$ are equivalent.
\end{definition}

\begin{restatable}{proposition}{thmequivsafe}
Two valid merge sequences $\merges$, $\merges' \in \validmerges$ are equivalent, i.e., $\merges \equiv \merges'$, if and only if there exists a safe permutation $\pi$ such that $\pi(\merges) = \merges'$.
\end{restatable}
\begin{proof}
The proof can be found in \Cref{sec:proofs}.
\end{proof}
Following the previous example, it is easy to verify that $\langle \mop{a}{b}, \mop{d}{d}, \mop{c}{a} \rangle \equiv \langle \mop{a}{b}, \mop{c}{a}, \mop{d}{d} \rangle$.
In contrast to synthetic examples with a constrained alphabet of, e.g., $\{a, b, c\}$, far fewer merge conflicts arise in natural language.
We can leverage this to develop a faster algorithm that only explores paths that are not equivalent to each other.
We first define the concept of partial ordering between merges.

\begin{definition}
The \defn{merge partial ordering} $\merge' \gtrdot \merge''$ is defined as $\neg \text{conflicts}(\merge', \merge'') \wedge  \neg (|\yield(\merge')| < |\yield(\merge'')|) \wedge \neg (\yield(\merge') <_L \yield(\merge''))$ where $>_L$ is lexicographical ordering.
\end{definition}

All valid merge sequences are equivalent to some merge sequence which is partially ordered using $\gtrdot$ so that no neighbouring elements violate this partial ordering.
The brute-force algorithm works as depth-first search through an acyclic graph: each state corresponds to a unique sequence of merges and each transition corresponds to appending a merge to the end of the current state's merges.
For the improved version, we make sure that only sequences which are ordered using $\gtrdot$ are searched and the rest are pruned.
The pseudocode for the program is shown in \Cref{algo:exact_bpe_dynamic}.
Even though the runtime is still prohibitively slow for application, \Cref{fig:dfs_speed} demonstrates how much speed is gained over the brute-force version which explores all states.

\begin{algorithm}[htbp]
{\fontsize{11}{12}\selectfont
\begin{algorithmic}[1]
\State $q \gets \Call{Stack}{\,}$
\State $q.\Call{push}{\langle\rangle, \str}$
\State $\merges^*, \str^* \gets \langle \rangle, \str$
\While{$|q| \neq 0$}
    \State $\merges, \str \gets q.\textsc{pop}()$
    \State \textbf{if} $|\merges| = M$ \textbf{then} \textbf{continue}
    \For{$\merge \in \textsc{Pairs}(\str)$}
        \State\colorbox{\algcolorC}{
        \textbf{if} $|\merges| = 0 \vee \merge \gtrdot \merges_{-1}$
        }
            \State \hspace{\algorithmicindent}
            $\str' \gets \Call{SingleApply}{\str, \merge}$
            \State \hspace{\algorithmicindent}
            $\merges' \gets \merges \circ \merge$
            
            \State \hspace{\algorithmicindent}
            \textbf{if} $|\str'| < |\str^*|$
                \State \hspace{\algorithmicindent}\hspace{\algorithmicindent}
                $\merges^*, \str^* \gets \merges', \str'$ 
            \State \hspace{\algorithmicindent}
            \textbf{end if}
            
            \State \hspace{\algorithmicindent}
            $q.\textsc{push}(\merges', \str')$
        \State\colorbox{\algcolorC}{
        \textbf{end if}
        }
    \EndFor
\EndWhile
\State \Return $\merges^*, \str^*$
\end{algorithmic}
}
\caption{%
    Exact BPE with memoization guard.\newline
    Removing segments marked with \colorbox{\algcolorC}{\textcolor{\algcolorC}{X}} would result in the brute-force version.\newline
    \textbf{Inputs}: string $\str$, merge count $M$\newline
    \textbf{Output}: tokenized string $\str$, merge sequence $\merges$
}
\label{algo:exact_bpe_dynamic}
\end{algorithm}

\section{Conclusion}

In this paper, we developed the formalisms surrounding the training task of BPE, a very popular tokenization algorithm in NLP.
This allowed us to prove a lower bound on the compression utility by greedy BPE as $1-e^{-\sigmaopt}$.
We further analyzed the runtime of the naïve and faster greedy BPE algorithms and provided a speedup for finding an optimal BPE merge sequence.
Future works should focus on providing either formal guarantees for $\sigmaopt$ or studying $\sigmaopt'$ across natural languages.

\section{Limitations}

Our work has focused strongly on the formal aspects of BPE.
NLP practictioners should not be dissuaded from using BPE for subword tokenization, despite our presentation of examples where greedy BPE fails.
Indeed, in contrast to synthetic examples on toy alphabet, on real data we made an observation that greedy BPE may be close to optimal.\looseness=-1

\section*{Acknowledgements}

We would like to thank Andreas Krause and Giorgio Satta for discussing the proof of \Cref{threorem:greedy_almost_optimal}.
Clara Meister was supported by the Google PhD Fellowship. Juan Luis Gastaldi has received funding from the European Union's Horizon 2020 research and innovation programme under grant agreement No 839730.

\vspace{-2mm}

\bibliography{misc/bibliography}
\bibliographystyle{misc/acl_natbib}

\appendix

\onecolumn


\section{Proofs}
\label{sec:proofs}

Our proof of approximate optimality is based on the proof of greedily sequence maximizing submodular functions by \citet{alaei2010maximizing,zhang2015string}.
However, we leverage a problem-specific property, which we dub hiearchical submodularity. We restate the definition here for ease.

\hierarchicalsub*

\begin{lemma} 
\label{lem:big-proof-helper}
Let $\merges', \merges \in \validmerges$ be valid merge sequences. Then, there exists a merge $\nu$ in $\merges$ such that $\merges'\nu$ is a valid merge sequence and $\kappa_{\str}(\nu\mid\merges') \geq \frac{\kappa_{\str}(\merges\mid\merges')}{|\merges|}$. In words, the compression gain of some element in $\merges$ with respect to $\merges'$ is greater or equal to the average compression gain per element of $\merges$ with respect to $\merges'$
\end{lemma}

\begin{proof}
Let us choose on of the possible maximums, $t = \argmax_{1\leq t'\leq |\merges|} \kappa_{\str}(\merge_{t'} \mid \merges'\merges_{<t'})$.
Because we are taking the maximum, which is always equal to or greater than the average,\footnote{Proof of this algebraic statement is omitted for brevity.} then $\kappa_{\str}(\merge_t \mid \merges'\merges_{<t}) \geq \frac{1}{|\merges|} \sum_{t'=1}^{|\merges|} \kappa_{\str}(\merge_{t'} \mid \merges'\merges_{<t'})$.
Then, we have that either:

\begin{itemize}
    \item $\merges\merge_{t} \in \validmerges$, in which case the result follows by submodularity, or
    \item $\merges\merge_{t} \notin \validmerges$, in which case there exists a $\merge_{t'}$ such that:
    \begin{itemize}
        \item $ \merge_{t'} \subset \merge_{t}$
        \item $\merges'\merge_{t}' \in \validmerges$
        \item $\merge_{t'}$ in $\merges$
        \item $\kappa_{\str}(\merge_{t} \mid \merges'\merges_{<t}) \leq \kappa_{\str}(\merge_{t'} \mid \merges'\merges_{<t'}) \leq \kappa_{\str}(\merge_{t'} \mid \merges')$
    \end{itemize}
    In particular, all trivial submerges of $\merge_t$ (i.e., all submerges of $\merge_t$ whose constituents are in $\alphabet$) fulfill all four conditions: the first one by definition, the second by definiton of $\validmerges$, the third because $\merges \in \validmerges$, and the fourth by hierarchical submodularity (first inequality) and by submodularity (second inequality).
\end{itemize}

\end{proof}

We now proceed with the proof of approximate optimality of the greedy BPE merge sequence.

\greedyalmostoptimal*
\label{proof:greedy_almost_optimal}


\begin{proof}
We make use of the sequence $\mergesgreedy_{<M}$ (rather than $\mergesgreedy$) for reasons that will subsequently become clear.
From \Cref{lem:big-proof-helper}, we know that we can find  $\mergeopt_j$ such that $ \mergesgreedy_{<M}\mergeopt_j$ is a valid merge sequence and 
\begin{align}
\kappa(\mergeopt_j \mid \mergesgreedy_{<M}) &\geq \frac{1}{M}\kappa(\mergesopt \mid \mergesgreedy_{<M})\label{eq:chiliastic_0}
\end{align}

From the greedy property of $\mergesgreedy$, we know:
\begin{align}
\kappa(\mergegreedy_M \mid \mergesgreedy_{<M}) &\geq \kappa(\mergeopt_j \mid \mergesgreedy_{<M}) \\ 
\kappa(\mergegreedy_M \mid \mergesgreedy_{<M})
&\geq \frac{1}{M} \kappa(\mergesopt \mid \mergesgreedy_{<M}) & \prooftext{(from Eq.~\ref{eq:chiliastic_0})}\\
\kappa(\mergesgreedy_{<M}\mergegreedy_M) - \kappa(\mergesgreedy_{<M})
&\geq \frac{1}{M} (\kappa(\mergesgreedy_{<M} \mergesopt) - \kappa(\mergesgreedy_{<M})) & \prooftext{(definition expansion)} \label{eq:omegoid_0}
\end{align}

Now from backward curvature (\Cref{def:total_backward_curvature}) and by substituting $\mergesgreedy_{<M}$ for the prefix sequence:
\begin{align}
\sigmaopt &\geq 1- \frac{\kappa(\mergesgreedy_{<M}\mergesopt)-\kappa(\mergesopt)}{\kappa(\mergesgreedy_{<M})} \\
\sigmaopt\kappa(\mergesgreedy_{<M}) &\geq \kappa(\mergesgreedy_{<M})- \kappa(\mergesgreedy_{<M}\mergesopt)+\kappa(\mergesopt) \\
\kappa(\mergesgreedy_{<M}\mergesopt)-\kappa(\mergesgreedy_{<M}) &\geq \kappa(\mergesopt) - \sigmaopt\kappa(\mergesgreedy_{<M}) 
\end{align}

Applying this result to the right-hand side of \Cref{eq:omegoid_0}, we obtain the following:
\begin{align}
\kappa(\mergesgreedy_{<M}\mergegreedy_M) - \kappa(\mergesgreedy_{<M}) &\geq \frac{1}{M} (\kappa(\mergesopt) - \sigmaopt \kappa(\mergesgreedy_{<M})) & \prooftext{(total backward curvature)} \\
\kappa(\mergesgreedy) - \kappa(\mergesgreedy_{<M}) &\geq \frac{1}{M} (\kappa(\mergesopt) - \sigmaopt \kappa(\mergesgreedy_{<M})) & \prooftext{(definition)}
\end{align}
\begin{align}
\kappa(\mergesgreedy) &\geq \frac{1}{M} (\kappa(\mergesopt) - \sigmaopt \kappa(\mergesgreedy_{<M})) + \kappa(\mergesgreedy_{<M}) \hspace{-2cm} & \prooftext{(total backward curvature)} \\
%
&\geq \frac{1}{M} \kappa(\mergesopt) + \left(1-\frac{\sigmaopt}{M}\right) \kappa(\mergesgreedy_{<M}) \hspace{-2cm} & \prooftext{(algebraic manipulation)}\\
&\geq \frac{1}{M} \kappa(\mergesopt) \sum_{i=0}^{M-1} \left( 1- \frac{\sigmaopt}{M} \right)^i
& \prooftext{(recursive substitution of $\kappa(\mergesgreedy_{<i})$)} \\
&=\frac{1}{\sigmaopt} \left( 1- \left( 1-\frac{\sigmaopt}{M}\right)^M \right) \kappa(\mergesopt) \hspace{-3cm}
& \prooftext{(geometric sum)} \\
&=\frac{1}{\sigmaopt} \left( 1- \left( 1-\frac{\sigmaopt}{M}\right)^\frac{M}{\sigmaopt} \right)^{\sigmaopt} \kappa(\mergesopt) \hspace{-3cm}
& \prooftext{(preparation)}
\end{align}
We substitute $x=\frac{M}{\sigmaopt}$ in the inequality.
From $x> 0 \Rightarrow \left(1-\frac{1}{x}\right)^x \leq \frac{1}{e}$, we obtain and arrive at
\begin{align}
\kappa(\mergegreedy)
&\geq \frac{1}{\sigmaopt} \left( 1 - e^{-\sigmaopt}\right)
\end{align}
\end{proof}

\thmequivsafe*

\begin{proof}
\noindent$\Rightarrowbracketed$
We prove the first implication through contrapositive, i.e., we show that if there does \textit{not} exist such a safe permutation $\pi$, then the merge sequences are \textit{not} equivalent.
By supposition, all non-safe permutations mapping $\merges$ to $\merges'$ either have a conflict or do not preserve validity.
We handle each case separately.
\begin{itemize}
\item \textbf{Case 1:}
Suppose that the permutation $\pi$ re-orders two conflicting merges $\merge$ and $\merge'$.
By the definition of a conflict, $\merge$ has yield $\str a$ and $\merge'$ has yield $a \str'$ for $a \in \alphabet$ and $\str, \str' \in \alphabet^*$. 
Now, note the bracketing string $\str a \str'$ will be different under the original and permuted merge sequence.
\item \textbf{Case 2:}
Suppose that the permutation $\pi$ does not preserve validity.
Then, there exists a merge $\merge = (\merge', \merge'')$ such that either $\merge'$ or $\merge''$ occurs \textit{after} $\merge$ in the merge sequence. 
This also results in a different bracketing.
\end{itemize}

\noindent$\Leftarrowbracketed$
Next, we want to show the converse, i.e., for any safe permutation $\pi$, we have $\merges \equiv \pi(\merges)$. 
Let $\merges = \langle \merge_1, \ldots, \merge_N\rangle$ be a merge sequence of length $N$, and let $\pi$ be a safe permutation.
We proceed by induction on the $n$.
\begin{itemize}
\item \textbf{Base Case:}
Since $\pi$ is safe, then for $\mop{a}{b} = \pi(\merges)_1$, $a$ and $b$ are necessarily characters in $\alphabet$.
\item \textbf{Inductive Step}:
Suppose for $k = n-1$, $\pi(\merges)_{\leq k}$ applies merges which are applied by $\merges$.
We then show $\pi(\merges)_n$ also applies the same merges as $\merges$.
Consider $\pi(\merges)_n = (\merge_m, \merge_{m'})$; since $\pi$ is safe, both $\merge_m$ and $\merge_{m'}$ already exist in $\enc_{\merges_{\leq n}}(\str)$.
Moreover, since there are no conflicts, applying $\pi(\merge)_n$ results in the same encoded sequence.
\end{itemize}
\vspace{-5mm}
\end{proof}

\validmergessize*

\begin{proof}
On one hand, we  note that we have an upper bound of $N-1$ possible merges that can occupy the first element of the sequence, assuming every symbol in $\str$ is distinct. 
Next, we have $N-2$ possible merges that can occupy the second element of the sequence, again, assuming every symbol in $\str$ is distinct. 
Continuing this pattern, we arrive at a simple upper bound on the number of merges $\prod_{m=0}^{M-1} \left(N - 1 - m\right)$.
This quantity is recognizable as a falling factorial, which gives us the closed form $\frac{\left(N - 1\right)!}{\left(N - M - 2\right)!}$.
This can be trivially bounded by $N^M$.
However, on the other hand, we know a valid merge sequence can produce merges with a yield up to length $M$, and there are ${\alphabet^{\leq M} \choose M}$ unique sequences.
We can upper-bound the number of valid merge sequences by the total number of all possible merge sequences, of which there are $M!$.
The size of $\alphabet^{\leq M}$ is the sum $|\alphabet|^1 + |\alphabet|^2 + \ldots |\alphabet|^M$ which is less than $M|\alphabet|^M$.
Again, with $M!$, this leads to the falling factorial $\frac{(M|\alphabet|^M)!}{(M|\alphabet|^M-M)!}$ which we can upper bound by $(M|\Sigma|^M)^M$ which is in $\bigO{|\Sigma|^{2M}}$.
Taking the min of these two upper bounds gives us the overall upper bound.
\end{proof}

\section{BPE Modifications}

In this section, we describe multiple modifications to the greedy BPE algorithm which speed up the runtime.
We do not address popular heuristic modifications such as lowercasing the text or adding 20\% of the most frequent words to the subword dictionary.

\subsection{(Not) Merging Space}

Currently, spaces are treated as any other characters and are allowed to be part of merges.
Therefore in the string \textit{``not{\spacesymbol}that{\spacesymbol}they{\spacesymbol}watch{\spacesymbol}the{\spacesymbol}watch''} the first merge is \textit{\mopt{\spacesymbol}{t}} and the string looks as \textit{``not\mop{\spacesymbol}{t}hat\mopt{\spacesymbol}{t}hey watch\mopt{\spacesymbol}{t}he watch''}.
The next merge may be across tokens: \textit{\mopt{t}{\mopt{\spacesymbol}{t}}}.
This is not desirable if we want only want to split tokens into subwords (i.e. use merges that do not contain spaces).

Furthermore, in \Cref{sec:greedy} we are duplicating work by computing pair frequencies and merges multiple times across the same tokens that occur multiple times (see previous string example).
In practice (\Cref{tab:dataset_statistics}), only 1.5\% of all tokens are unique.
We may then speed up our computation by considering only unique tokens.
Therefore, the new runtime complexity is $\bigO{V\cdot |\str_u|}$ where $\str_u = \{ t \mid \text{token } t \in \str\}$ which is $\frac{|\str|}{|\str_u|}\times$ faster.

\subsection{Non-iterative BPE}

A popular implementation of BPE-like algorithm in Python\footnote{\href{https://pypi.org/project/bpe/}{pypi.org/project/bpe}} uses a different speed-up mechanism to avoid $\bigO{N V}$ runtime.
This is done by:
\begin{enumerate}[noitemsep]
\item collecting all possible merges observed in the data up until some maximum yield size which determines the maximum subword size, such as 5 and
\item taking top-$M$ frequent pairs as part of the subword dictionary.
\end{enumerate}
Note that because of hiearchical submodularity (\Cref{def:hierarchical-submodular}), this will produce valid merges.
This is because if $\merge = \mop{\merge'}{\merge''}$ is chosen, so must $\merge'$ and $\merge''$ because they have at least the same frequency as $\merge$.
For example, for $abcabcd$, and maximum yield width 3, the merges would be $\mop{a}{b}, \mop{\mop{a}{b}}{c}, \mop{b}{c}, \mop{a}{\mop{b}{c}}, \ldots$.
The runtime of this is $\bigO{|\str| \log M}$ because we are scanning the whole string and at each point are modifying maximum heap.

However, it is easy to see that this approximation algorithm is not bounded.
For a constant maximum yield width of $w$, consider $\str = a^{wn}$ and $V = w+k$.
The shortest possible output of this algorithm will be $\merge^n$.
However, an optimal merge sequence can perform additional merge sequences, therefore producing $\vnu^{\frac{n}{2^k}}$.
The compressions are $wn - n$ and $wn - \frac{n}{2^k}$ and the ratio $\frac{wn - n}{wn - \frac{n}{2^k}}$ with lower bound of $0$ as supremum.
This means that we can construct adversarial example for which the compression given by this algorithm is arbitrarily suboptimal.



\section{Additional Experimental Details and Results}
\begin{table}[htbp]
\centering
\begin{tabular}{lcc}
\toprule
Sentence count (train) & 13M+13M \\
Sentence count (dev \& test) & 1M+1M \\
Total words & 324M \\
Unique words & 5M \\
Average sentence length (words) & 12 \\
\bottomrule
\end{tabular}
\caption{Overview of the used portion of the English-German CommonCrawl dataset \citep{elkishky_ccaligned_2020}.}
\label{tab:dataset_statistics}
\end{table}

\begin{figure}[htbp]
\centering
\includegraphics[width=0.55\linewidth]{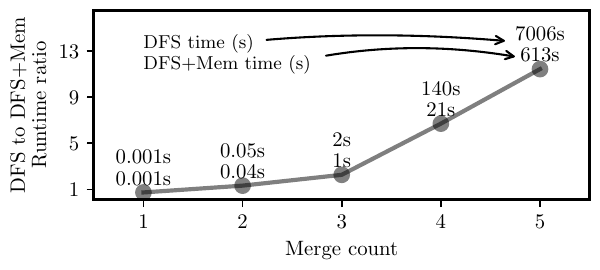}
\vspace{-4mm}
\caption{Comparison of runtimes for brute-force DFS and DFS with memoization.
Values above 1 correspond to DFS+memoization being $\times$ faster than DFS.
Points show average\footnotemark of runs on 5 different input strings (each 2 randomly sampled English sentences of 64 characters).}
\label{fig:dfs_speed}
\end{figure}
\footnotetext{Time measured on desktop AMD Ryzen 9 5900X.}

\clearpage

\begin{figure}[H]
\vspace{1cm}
\begin{minted}[fontsize={\fontsize{9.5}{9}\selectfont},linenos,xleftmargin=7mm]{python}
from collections import Counter, defaultdict
from typing import Union, Tuple, List

def fixed_pair_freqs(xs: Union[str, List]):
    pairs = defaultdict(int)
    prev_pair = None
    for (x, y) in zip(xs, xs[1:]):
        # increment only if the prev suffix does not match prefix
        # otherwise wrong estimate on `aaa`
        if (x,y) != prev_pair:
            pairs[x, y] += 1
            prev_pair = (x, y)
        else:
            # make sure to clear it so that `aaaa` is counted twice
            prev_pair = None
            
    pairs = list(pairs.items())
    pairs.sort(key=lambda x: x[1], reverse=True)
    return pairs

def bpe(xs: Union[str, List], V: int):
  for _ in range(V):
    top_pair = fixed_pair_freqs(xs)[0]
    xs = merge(list(xs), top_pair)
  return xs

def merge(xs: List, pair: Tuple):
  ys = []
  while xs:
    if tuple(xs[:2]) == pair:
      ys.append(pair)
      xs = xs[2:]
    else:
      ys.append(xs.pop(0))
  return ys 
\end{minted}
\captionof{mycode}{An implementation of \citeposs{sennrich-etal-2016-neural} greedy algorithm for BPE in Python with overlap-adjusted pair counts.}
\label{code:fixed_bpe}
\end{figure}

\end{document}